\documentclass[letterpaper, 10pt, conference]{ieeeconf}
\IEEEoverridecommandlockouts 
\pdfoutput=1

\usepackage{multicol}
\usepackage[bookmarks=true]{hyperref}

\usepackage{xparse} 
\usepackage{amsmath} 
\usepackage{amssymb}
\usepackage{amsfonts} 
\usepackage{dsfont}
\usepackage{graphicx}
\usepackage[caption=false]{subfig}

\usepackage{url}
\usepackage{booktabs}
\usepackage[para]{threeparttable}
\usepackage{hyperref}
\hypersetup{colorlinks,linkcolor={blue},citecolor={green},urlcolor={blue}}

\usepackage{amsthm}
\usepackage{cleveref}
\usepackage{xcolor}
\usepackage{framed}\usepackage{xcolor,cancel}

\colorlet{shadecolor}{gray!10}
   
\newtheorem{problem}{Problem}

\newtheorem{theorem}{Theorem}
\newtheorem{proposition}{Proposition}
\newtheorem{corollary}{Corollary}
\newtheorem{lemma}{Lemma}

\let\oldthempfootnote\thempfootnote
\def\thempfootnote{\text{\oldthempfootnote}}
\usepackage{color}
\usepackage{balance}

\newcommand{\norm}[2]{\|#1\|_{#2}}
\newcommand{\frob}[1]{\|#1\|_{\scriptscriptstyle F}}

\NewDocumentCommand\bbm{}{ \begin{bmatrix} }
\NewDocumentCommand\ebm{}{ \end{bmatrix} }
\NewDocumentCommand\Vector{m}{ \boldsymbol{\mathbf{#1}} }
\NewDocumentCommand\Matrix{m}{ \boldsymbol{\mathbf{#1}} }

\NewDocumentCommand\Trace{m}{ \mathrm{tr}\left(#1\right) }

\NewDocumentCommand\Norm{m}{\left\Vert#1\right\Vert }

\NewDocumentCommand\Vectorize{m}{ \mathrm{vec}\left(#1 \right) }

\NewDocumentCommand\Real{}{ \mathbb{R} }

\NewDocumentCommand\Sym{}{ \mathbb{S} }

\NewDocumentCommand\LieGroupSO{m}{ \mathrm{SO}(#1) }
\NewDocumentCommand\LieGroupO{m}{ \mathrm{O}(#1) }

\NewDocumentCommand\LieGroupSE{m}{ \mathrm{SE}(#1) }

\NewDocumentCommand\Identity{}{ \Matrix{I} }

\NewDocumentCommand\CoordinateFrame{m}{ \underrightarrow{\Matrix{\mathcal{F}}}_{#1} }

\NewDocumentCommand\T{}{\mathsf{T}}

\NewDocumentCommand\LagrangianHessian{}{\Matrix{\mathcal{Q}}}
\NewDocumentCommand\Span{}{ \mathrm{span} }
\NewDocumentCommand\Col{}{ \mathrm{col} }
\NewDocumentCommand\Rank{}{ \mathrm{rank} }

\title{\Large \bf Certifiably Optimal Monocular Hand-Eye Calibration}

\author{Emmett Wise$^{\dagger}$, Matthew Giamou$^{1,\dagger}$, Soroush Khoubyarian, Abhinav Grover, and Jonathan Kelly
\thanks{$^\dagger$ Denotes equal contribution.}
\thanks{All authors are with the Space \& Terrestrial Autonomous Robotic Systems (STARS) Laboratory at the University of Toronto Institute for Aerospace Studies (UTIAS), Toronto, Canada. {\tt\footnotesize <firstname>.<lastname>@robotics.utias.utoronto.ca}}
\thanks{$^{1}$ Vector Institute Postgraduate Affiliate and RBC Fellow.}}

\begin{document}
\maketitle

\begin{abstract}
Correct fusion of data from two sensors requires an accurate estimate of their relative pose, which can be determined through the process of \emph{extrinsic calibration}. When the sensors are capable of producing their own egomotion estimates (i.e., measurements of their trajectories through an environment), the `hand-eye' formulation of extrinsic calibration can be employed. In this paper, we extend our recent work on a convex optimization approach for hand-eye calibration to the case where one of the sensors cannot observe the scale of its translational motion (e.g., a monocular camera observing an unmapped environment). We prove that our technique is able to provide a certifiably globally optimal solution to both the known- and unknown-scale variants of hand-eye calibration, provided that the measurement noise is bounded. Herein, we focus on the theoretical aspects of the problem, show the tightness and stability of our convex relaxation, and demonstrate the optimality and speed of our algorithm through experiments with synthetic data.
\end{abstract}

\section{Introduction}
\label{sec:intro}

Many autonomous mobile robots perceive their environments by fusing noisy measurements from multiple sensors. 
While certain high-end devices such as 3D lidar units are able to provide a fairly complete perception solution on their own, the use of multiple sensors confers the ability to leverage complementary modalities (e.g., rich colour camera images and high-rate inertial measurements) to improve reliability and robustness.
For safety-critical applications, these properties are often a necessity, but come at the cost of greater complexity: the quality of the robot's map and trajectory estimates depends directly on accurate knowledge of the rigid-body transformation between each pair of sensor reference frames (see \Cref{fig:sensor_geometry}).

The process of determining the rigid transformation between sensor reference frames is typically referred to as \emph{extrinsic calibration}. 
 Robot manufacturers often provide an estimate of this transformation by performing factory calibration using expensive custom instruments. 
However, many end-users will either augment their platforms with additional sensors or modify the spatial configuration of existing sensors to suit their specific needs.
Additionally, operation of the robot inevitably leads to unintentional changes to the extrinsic transformation, either due to gradual structural and material effects like thermal expansion and metal fatigue, or from impacts caused by collisions.
Since autonomous robots operating in the field do not have access to factory calibration systems, automatic extrinsic calibration algorithms have received a great deal of attention from researchers. 
Proposed methods vary in both their generality (e.g., specificity of the combination of sensors involved), and in the assumptions they make about the calibration environment (e.g., requiring the use of inexpensive calibration `targets' like checkerboards \cite{qilongzhang_extrinsic_2004} or common architectural features like room corners \cite{gomez-ojeda_extrinsic_2015}). 

\begin{figure}[t]
	\centering
	\includegraphics[width=0.98\columnwidth]{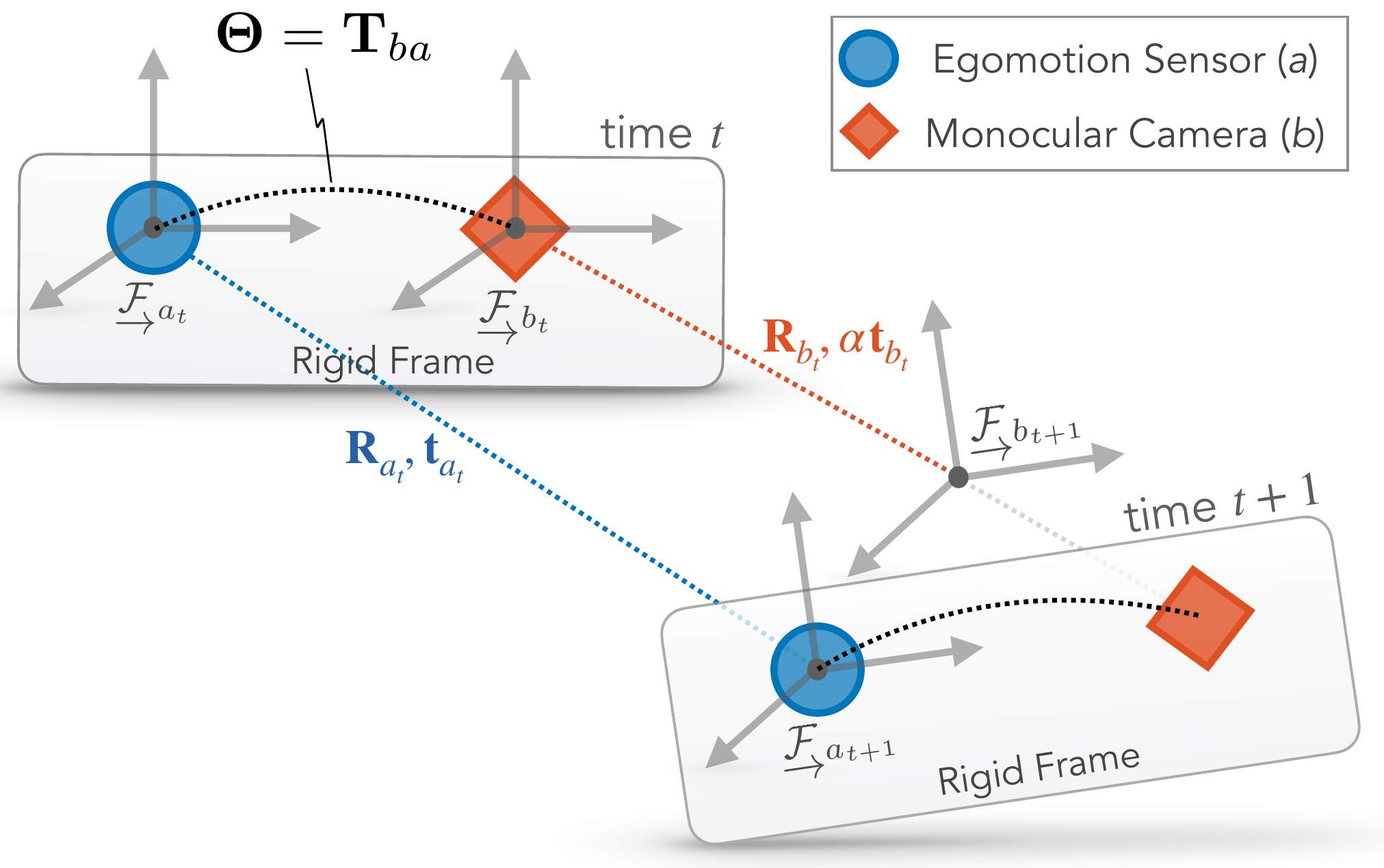}
	\caption{We perform extrinsic calibration between a sensor ($a$) that is able to provide egomotion estimates with known scale and a monocular camera ($b$) that cannot observe the scale, $\alpha$, of its translational motion. A quadratically-constrained quadratic program (QCQP) formulation of hand-eye calibration is relaxed into a convex semidefinite program (SDP) which can be efficiently solved to yield a globally optimal estimate of the extrinsic transformation matrix $\Matrix{\Theta}$. The use of a convex relaxation permits \emph{certification} of global optimality; we prove that the relaxation is guaranteed to be tight when measurement noise is bounded.}
	\label{fig:sensor_geometry}
	\vspace{-0.4cm}
\end{figure}

The term \emph{hand-eye calibration} refers to determining the extrinsic transformation between a robotic manipulator and a sensor (typically a camera) mounted on or held by the manipulator. 
However, the common ${\Matrix{A}\Matrix{X} = \Matrix{X}\Matrix{B}}$ formulation of this problem can be applied to any sensor that can estimate its egomotion, including stereo cameras, 3D lidar units, and GNSS-INS devices \cite{brookshire2012automatic, brookshire2013extrinsic}.
Throughout this work, we refer to the sensor-agnostic ${\Matrix{A}\Matrix{X} = \Matrix{X}\Matrix{B}}$ problem illustrated in \Cref{fig:sensor_geometry} as hand-eye calibration. 
Crucially, we extend the certifiably optimal hand-eye calibration method in \cite{giamou2019_certifiably} so that one of the sensors need not measure the \emph{scale} of its translation correctly (in a metric sense).  
This case is important because it permits the use of our algorithm with sensor pairs that include a monocular camera (observing an unmapped environment) \cite{chiuso2002_structure}. 
Our problem formulation requires only that there are sufficient geometric features in the scene for each sensor to produce a (scaled, in the case of the camera) egomotion estimate. The main contributions of our work are:

\begin{enumerate}
\item an extension of the egomotion-based calibration problem formulation in \cite{giamou2019_certifiably} to include a monocular camera as one of the sensors,
\item a fast and certifiably optimal solution method using a convex semidefinite programming (SDP) relaxation,
\item a proof of the global optimality of our method in the presence of sensor measurement error, and
\item an open source empirical evaluation of our algorithm on synthetic data.\footnote[2]{See \url{https://github.com/utiasSTARS/certifiable-calibration} for Python code implementing our algorithm and experiments.}
\end{enumerate}

\noindent Our proof of global optimality in the presence of noise, also called SDP `stability' \cite{cifuentes_local_2018}, is the first of its kind for a hand-eye calibration algorithm. In \Cref{sec:related_work} we survey traditional extrinsic calibration methods that do not have this guarantee and do not provide certificates of (sub)optimality, making them susceptible to poor performance in the presence of noise. Our notation and problem formulation are introduced in \Cref{sec:formulation}, followed by our proposed solution method in \Cref{sec:solution}. In \Cref{sec:theory}, our theoretical results are presented. \Cref{sec:experiments} presents experiments validating our theoretical claims. We offer concluding remarks and ideas for future work in \Cref{sec:conclusion}.

\section{Related Work} 
\label{sec:related_work}

We begin with a brief survey of hand-eye calibration algorithms. This well-known problem has been studied since the 1980s and we direct readers to the short literature reviews in \cite{heller2014hand} and \cite{hu_novel_2019} for more information on recent approaches. In \Cref{sec:certifiable_algorithms}, we summarize the state of the art in the application of convex relaxation techniques to estimation problems in computer vision and robotics.  

\subsection{Hand-Eye Calibration}
\label{sec:handeye_algorithms}

Much of the early research on hand-eye calibration explored fast, closed-form solutions appropriate for the limited computational resources available at the time. These methods are tailored to the literal hand-eye case, where a robot arm with accurate forward kinematics moves a camera, usually in front of a known calibration target \cite{tsai_new_1989}. 
A dual quaternion-based formulation is explored in \cite{daniilidis_hand-eye_1999} and the advantages of coupling translation and rotation estimation are reviewed. 
The experimental investigation in \cite{horaud1995hand} concludes that nonlinear optimization approaches that couple translation and rotation estimation, like the algorithm developed in this paper, provide more accurate solutions in the presence of noise than the simpler but decoupled, closed-form methods.
Our problem formulation is similar to the one in \cite{andreff_robot_2001}, where observability criteria and different solution methods are presented in a systematic fashion. The unknown scale case is also studied in \cite{wei2018calibration}, where a second-order cone programming solution without optimality guarantees is proposed.

Recent research extends the hand-eye formulation to generic robotic platforms (e.g., self-driving vehicles \cite{walters_robust_2019}) and noisy egomotion measurements. 
Principled probabilistic (i.e., maximum likelihood) formulations of hand-eye calibration are the subject of \cite{brookshire2012automatic} and \cite{brookshire2013extrinsic}. 
A similar approach is applied to a related multi-robot calibration problem in \cite{ma_probabilistic_2016}. 
Our technique eschews a probabilistic cost function in order to leverage the simplicity of a classic geometric formulation, however combining the two approaches is a promising future direction. 

\subsection{Convex Relaxations and Certifiable Algorithms} 
\label{sec:certifiable_algorithms}

Convex SDP relaxations have become a popular solution method for a number of geometric estimation problems including rotation averaging \cite{fredriksson2012simultaneous, eriksson2018rotation, yang_teaser_2020}, SLAM \cite{rosen2016se, briales2017cartan, fan_efficient_nodate}, registration \cite{briales2017convex, olsson2008solving}, relative pose estimation \cite{briales_certifiably_2018, garcia-salguero_certifiable_2020, zhao_efficient_2019},  and hand-eye calibration \cite{heller2014hand, giamou2019_certifiably}. 
In \cite{yang_teaser_2020}, the authors promote a `certifiable perception manifesto' in which they argue that the global (sub)optimality certificates provided by these methods are essential for autonomous robotic systems that must operate reliably and safely in human-centric environments.
Similarly, we believe that accurate and automatic extrinsic calibration is essential for safe operation, especially for robots expected to function autonomously for long periods of time. 
Along with \cite{heller2014hand} and \cite{giamou2019_certifiably}, our approach is a \emph{certifiable algorithm} for hand-eye calibration. Our method advances the state of the art by including the monocular camera case and by providing the analytic global optimality guarantees presented in \Cref{sec:theory}.

\section{Problem Formulation}
\label{sec:formulation}

In this section, we introduce our notation and formulate our problem by extending the known-scale case found in \cite{giamou2019_certifiably}.

\subsection{Notation}
\label{sec:notation}

Lower case Latin and Greek letters (e.g., $a$ and $\alpha$) represent scalar variables, while boldface lower and upper case letters (e.g., $\Vector{x}$ and $\Matrix{\Theta}$) represent vectors and matrices respectively. 
The superscript in $\Matrix{A}^{(i)}$ indicates the $i$th column of the matrix $\Matrix{A}$. A three-dimensional reference frame is designated by $\CoordinateFrame{}$.
The translation vector from point $b$ (often a reference frame origin) to $a$, represented in $\CoordinateFrame{b}$, is denoted $\Vector{t}_b^{ab}$.
We denote rotation matrices by $\Matrix{R}$; for example, $\Matrix{R}_{ba} \in \text{SO}(3)$ describes the rotation from $\CoordinateFrame{a}$ to $\CoordinateFrame{b}$. We reserve $\Matrix{T}$ for $\LieGroupSE3$ transformation matrices; for example, $\Matrix{T}_{ba}$ is the the matrix that defines the rigid-body transformation from frame $\CoordinateFrame{a}$ to $\CoordinateFrame{b}$.
The binary $\otimes$ operator denotes the matrix Kronecker product.
The unary operator $\wedge$ acts on $\Vector{r} \in \Real^3$ to produce a skew symmetric matrix such that $\Vector{r}^\wedge\Vector{s}$ is equivalent to the cross product $\Vector{r}\times \Vector{s}$.
The function $\Vectorize{\cdot}$ unwraps the columns of matrix $\Matrix{A} \in \Real^{m\times n}$ into a vector $\Vector{a} \in \Real^{mn}$.

\subsection{Rigid Rotations and Transformations}
As discussed in Section \ref{sec:notation}, the rotation between two reference frames is represented as an element of $\LieGroupSO{3}$. The special orthogonal group $\LieGroupSO{3}$ can be described in terms of quadratic constraints \cite{ivancevic2011lecture}:
\begin{equation}\label{eq:SO3con}
\begin{aligned}
\LieGroupSO{3}: \; & \Matrix{R} \in \Matrix{\Real}^{3 \times 3}, \\
 \text{s.t.}{} \quad & \Matrix{R}^\T\Matrix{R}=\Matrix{R}\Matrix{R}^\T =\Identity, \\
 & \Matrix{R}^{(i)} \times \Matrix{R}^{(j)} = \Matrix{R}^{(k)}, \; i,j,k \in \text{cyclic}(1,2,3).   
\end{aligned}
\end{equation}

Similar to rotations, the rigid transformation between two rotated and translated reference frames $\CoordinateFrame{b}$ and $\CoordinateFrame{a}$ is an element of the special Euclidean group $\LieGroupSE3$ and is also defined by quadratic constraints \cite{ivancevic2011lecture}:
\begin{equation}
\begin{aligned}
\LieGroupSE{3}: \; & \Matrix{T}_{ba} \in \Matrix{\Real}^{4 \times 4}, \\[1mm]
 \text{s.t.} \; & \Matrix{T}_{ba} = \bbm \Matrix{R}_{ba} & \Vector{t}_b^{ab} \\ \Vector{0}_{1\times3} & 1 \ebm, \\[1mm]
 & \Matrix{R}_{ba} \in \LieGroupSO{3},\; \Vector{t}_b^{ab} \in \Matrix{\Real}^3.
\end{aligned}
\end{equation}  

\subsection{Extrinsic Calibration}

We denote two rigidly connected sensor reference frames at time $t$ as $\CoordinateFrame{a_t}$ and $\CoordinateFrame{b_t}$. 
Additionally, we introduce an arbitrary fixed inertial world frame $\smash{\CoordinateFrame{w}}$. 
Since the sensors are rigidly connected, there is some constant $\Matrix{\Theta} \triangleq \Matrix{T}_{ba} \in \LieGroupSE{3}$ that describes the transformation between $\CoordinateFrame{b_t}$ and $\CoordinateFrame{a_t}$:
\begin{equation} \label{eq:extrinsic_calibration}
\Matrix{T}_{w a_t} = \Matrix{T}_{w b_t}\Matrix{\Theta} \quad \forall \; t.
\end{equation}
With the basic algebraic manipulation presented in \cite{giamou2019_certifiably}, we can derive the central `${\Matrix{AX} = \Matrix{XB}}$' equation used in hand-eye calibration (see \Cref{fig:sensor_geometry}):
\begin{equation}
\Matrix{T}_{b_tb_{t+1}}\Matrix{\Theta} = \Matrix{\Theta}\Matrix{T}_{a_ta_{t+1}}.
\end{equation}

\subsection{Monocular Camera Egomotion}

Monocular camera measurements in an unmapped environment can be used estimate camera pose up to scale \cite{chiuso2002_structure}.
Thus, we can define the camera's incremental egomotion by
\begin{equation} \label{eq:cam_ob}
\Matrix{T}_{b_tb_{t+1}} = \bbm \Matrix{R}_{b_tb_{t+1}} & \alpha\Vector{t}^{b_{t+1}b_t}_{b_t} \\[1mm] \Vector{0}_{1\times3} & 1 \ebm,
\end{equation}
where $\alpha$ is an unknown (unobservable) scaling factor. 

\subsection{QCQP Formulation}

Assuming that the second sensor is able to measure its egomotion, $\Matrix{T}_{a_ta_{t+1}}$, the extrinsic calibration problem can be described as one of minimizing the magnitude over the error matrices from $T$ time steps:
\begin{equation} \label{eqn:error_eqn}
\Matrix{E}_t = \Matrix{\Theta}_{ba}\Matrix{T}_{a_ta_{t+1}} - \Matrix{T}_{b_tb_{t+1}}\Matrix{\Theta}_{ba} \in \LieGroupSE{3},\ t = 1, ..., T.
\end{equation}
To make our notation more compact, we denote $\Matrix{R}_{ba}$ as $\Matrix{R}$, $\Vector{t}_b^{ab}$ as $\Vector{t}$, $\Matrix{R}_{a_ta_{t+1}}$ as $\Matrix{R}_{a_t}$, and $\Vector{t}_{a_t}^{a_{t+1}a_t}$ as $\Vector{t}_{a_t}$ (and likewise for $b$'s motion). 
Minimizing the sum of the squared Frobenius norm of each $\Matrix{E}_t$ produces a QCQP:
\begin{problem}{QCQP formulation of Hand-Eye Calibration.}\label{prob:qcqp_formulation}
	\begin{equation}
\begin{aligned}
\min_{\Matrix{R}, \Vector{t}, \alpha} \; & J_{\Matrix{t}} + J_{\Matrix{R}}, \\
\text{\emph{s.t.}} \; & \Matrix{R} \in \LieGroupSO{3},
\end{aligned}
\end{equation}
where
\begin{equation*}
\begin{aligned}
J_{\Matrix{R}} &= \sum_{t=1}^T \frob{\Matrix{R}\Matrix{R}_{a_t} - \Matrix{R}_{b_t}\Matrix{R}}^2, \\
J_{\Matrix{t}} &= \sum_{t=1}^T \norm{\Matrix{R}\Vector{t}_{a_t} + \Vector{t}
 - \Matrix{R}_{b_t}\Vector{t} - \alpha\Vector{t}_{b_t}}{2}^2.
\end{aligned}
\end{equation*}
\end{problem}
\noindent Using the identity $\Matrix{A}\Matrix{X}\Matrix{B} = (\Matrix{B}^\T \otimes \Matrix{A})\Vectorize{\Matrix{X}}$ \cite{fackler2005notes}, the cost function of \Cref{prob:qcqp_formulation} can be converted into a standard quadratic form: 
\begin{equation}
\begin{aligned}
\Matrix{M}_{\Matrix{R}_t} &= \bbm \Matrix{0}_{9\times4} & \Matrix{R}_{a_t}^\T \otimes \Identity_{3\times3} - \Identity_{3\times3} \otimes \Matrix{R}_{b_t} \ebm, \\
J_{\Matrix{R}} &= \sum_{t=1}^T \Vector{x}^\T\Matrix{M}_{\Matrix{R}_t}^\T\Matrix{M}_{\Matrix{R}_t}\Vector{x}, \\
\Matrix{M}_{\Vector{t}\alpha_t} &= \bbm \Identity_{3\times3} - \Matrix{R}_{b_t} & -\Vector{t}_{b_t} & \Vector{t}_{a_t}^\T\otimes\Identity_{3\times3} \ebm, \\ 
J_{\Matrix{t}} &= \sum_{t=1}^T \Vector{x}^\T\Matrix{M}_{\Vector{t}\alpha_t}^\T\Matrix{M}_{\Vector{t}\alpha_t}\Vector{x},\\
\Vector{x}^\T &= \bbm \Vector{t}^\T & \alpha & \Vector{r}^\T \ebm, \\
J_{\Matrix{R}} + J_{\Matrix{t}} &= \Vector{x}^\T\Matrix{Q}\Vector{x},
\end{aligned}
\end{equation}
where $\Vector{r}^\T = \Vectorize{\Matrix{R}}$ and the symmetric cost matrix $\Matrix{Q}$ can be subdivided into 
\begin{equation} \label{eq:cost_matrix}
	\Matrix{Q} = \bbm \Matrix{Q}_{\Vector{t}\alpha} & \Matrix{Q}_{\Vector{t}\alpha,\Vector{r}} \\ \Matrix{Q}_{\Vector{t}\alpha,\Vector{r}}^\T & \Matrix{Q}_{\Vector{r}} \ebm. \\
\end{equation}
Given an optimal rotation matrix $\Matrix{R}^\star$, the unconstrained optimal translation vector $\Vector{t}^\star$ and scale $\alpha^\star$ can be recovered by solving the linear system induced by \Cref{eq:cost_matrix}:
\begin{equation} \label{eq:recover_trans}
\bbm \Vector{t}^\star & \alpha^\star \ebm = -\Matrix{Q}_{\Vector{t}\alpha}^{-1}\Matrix{Q}_{\Vector{t}\alpha,\Vector{r}}\Vector{r}^\star.
\end{equation}
This allows us to use the Schur complement to reduce the cost matrix to one that does not include $\Vector{t}$ and $\alpha$ \cite{briales2017convex}:
\begin{equation}
    \tilde{\Matrix{Q}} = \Matrix{Q}_{\Vector{r}} - \Matrix{Q}_{\Vector{t}\alpha,\Vector{r}}^\T\Matrix{Q}_{\Vector{t}\alpha}^{-1}\Matrix{Q}_{\Vector{t}\alpha,\Vector{r}}.
\end{equation}
A reduced form of \Cref{prob:qcqp_formulation} can then be defined that only includes the rotation variable:
\begin{problem}{Reduced QCQP Formulation of Hand-Eye Calibration.}\label{prob:qcqp_reduced}
\begin{equation}
\begin{aligned}
\min_{\Vector{r}=\Vectorize{\Matrix{R}}} \; & \Vector{r}^\T\tilde{\Matrix{Q}}\Vector{r},\\
\text{\emph{s.t.}} \; & \Matrix{R} \in \LieGroupSO{3}.
\end{aligned}
\end{equation} 
\end{problem}

\subsection{Homogenization} \label{sec:homogenization}
In order to simplify the convex Lagrangian dual relaxation of \Cref{prob:qcqp_formulation} in \Cref{sec:solution}, the constraints given by (\ref{eq:SO3con}) can be homogenized with scalar variable $y$:
\begin{equation}\label{eq:homog_qcqp}
\begin{aligned}
\Matrix{R}^\T\Matrix{R} &= y^2\Identity, \\
\Matrix{R}\Matrix{R}^\T &= y^2\Identity, \\
\Vector{R}^{(i)} \times \Vector{R}^{(j)} &= y\Vector{R}^{(k)}, \; i,j,k \in \text{cyclic}(1,2,3), \\
y^2 &= 1.\\ 
\end{aligned}
\end{equation}
This forms a set of $22$ homogeneous quadratic equality constraints (six for each orthogonal constraint, three from each cyclic cross product, and one from the homogenizing variable $y$). 
This also requires augmenting the state to include $y$ such that $\tilde{\Vector{r}}^\T =[\Vectorize{\Matrix{R}}^\T \ y\,]$, and padding $\tilde{\Matrix{Q}}$ with zeros. 

\section{Solving the Non-Convex QCQP}
\label{sec:solution}

\Cref{sec:formulation} described the extrinsic calibration problem as a nonconvex homogeneous QCQP. 
To transform this problem into a convex SDP that is easier to solve, we derive its Lagrangian dual relaxation using the standard procedure outlined in \cite{boyd2004convex}: see \cite{giamou2019_certifiably} and \cite{Briales_2017} for detailed treatments of QCQPs similar to ours. 

\subsection{Lagrangian Dual}
\label{Dual}

 Using the homogenized $\LieGroupSO3$ constraints of \Cref{eq:homog_qcqp} from \Cref{sec:homogenization}, the Lagrangian function $L(\Vector{\tilde{r}}, \Vector{\nu})$ of \Cref{prob:qcqp_formulation} has the form
\begin{equation} \label{eq:lagrangian_dual}
\begin{aligned}
L(\Vector{\tilde{r}}, \Vector{\nu}) &= \nu_{y} +\Vector{\tilde{r}}^\T\Matrix{Z}\Vector{\tilde{r}}, \\
\Matrix{Z}(\Vector{\nu}) &= \tilde{\Matrix{Q}} + \Matrix{P}_1(\Vector{\nu}) + \Matrix{P}_2(\Vector{\nu}),
\end{aligned}
\end{equation}
where 
\begin{equation} \label{eq:dual_constraints}
\begin{aligned}
 \Matrix{P}_1(\Vector{\nu})  &= \bbm - \Matrix{\mathcal{V}}_1 \otimes \Identity_{3\times3} -\Identity_{3\times3} \otimes \Matrix{\mathcal{V}}_2 & \Matrix{0}_{9\times 1} \\ \Matrix{0}_{1\times 9} & \Trace{\Matrix{\mathcal{V}}_1} + \Trace{\Matrix{\mathcal{V}}_2} \ebm, \\[1mm]
\Matrix{P}_2(\Vector{\nu}) &=  \bbm \Matrix{0}_{3\times3} & -\Vector{\nu}_{ijk}^{\wedge} & \Vector{\nu}_{kij}^{\wedge} & -\Vector{\nu}_{jki}\\
 \Vector{\nu}_{ijk}^{\wedge} & \Matrix{0}_{3\times3} & -\Vector{\nu}_{jki}^{\wedge} & -\Vector{\nu}_{kij} \\
 -\Vector{\nu}_{kij}^{\wedge} & \Vector{\nu}_{jki}^{\wedge} & \Matrix{0}_{3\times3} & -\Vector{\nu}_{ijk} \\ 
 -\Vector{\nu}_{jki}^\T & -\Vector{\nu}_{kij}^\T & -\Vector{\nu}_{ijk}^\T & -\nu_{y} \ebm, \\[1mm]
 \Matrix{\mathcal{V}}_1 &, \, \Matrix{\mathcal{V}}_2 \in \Sym^3, \;
\Vector{\nu}_{ijk}, \, \Vector{\nu}_{jki}, \, \Vector{\nu}_{kij} \in \Vector{\mathbb{R}}_3,
\end{aligned}
\end{equation}
and where $\Sym^3$ is the set of all 3$\times$3 real symmetric matrices and $\Vector{\nu} \in \Real^{22}$ is a vector containing all dual variables.
Next, we minimize the Lagrangian function with respect to $\Vector{\tilde{r}}$: \begin{equation}
\min_{\Vector{\tilde{r}}}\; L(\Vector{\tilde{r}}, \Vector{\nu}) = \begin{cases} \nu_{y} & \Matrix{Z}(\Vector{\nu}) \succcurlyeq 0, \\
-\infty & \text{otherwise.}
\end{cases} 
\end{equation}
Finally, the Lagrangian dual problem is the following SDP:
\begin{problem}[Dual of \Cref{prob:qcqp_reduced}] \label{prob:dual_relaxation}
\begin{equation}
\begin{aligned}
\max_{\Vector{\nu}} \; & \nu_{y}, \\
\text{\emph{s.t.}} \; & \Matrix{Z}(\Vector{\nu}) \succcurlyeq 0,
\end{aligned}
\end{equation}
where $\Matrix{Z}(\Vector{\nu})$ is defined in \Cref{eq:lagrangian_dual,eq:dual_constraints}.
\end{problem}
\Cref{prob:dual_relaxation} can be efficiently solved with any generic interior-point solver for SDPs \cite{SeDuMi,andersen_mosek_2000,toh_sdpt3_1999}. 
Once we have found the optimal dual parameters $\Vector{\nu}^\star$, the primal solution can be found by examining the Lagrangian dual (\Cref{eq:lagrangian_dual}): because $\Matrix{Z}$ is positive semidefinite (PSD), the $\tilde{\Vector{r}}^\star$ that minimizes \Cref{eq:lagrangian_dual} lies in the nullspace of $\Matrix{Z}$ \cite{Briales_2017}. 
Since we enforce that $y = 1$, the optimal rotation is actually $\tilde{\Vector{r}}^\star/y^\star$, and $\Matrix{R}^\star$ can be recovered by horizontally stacking the columns of $\tilde{\Vector{r}}^\star/y^\star$, while $\Vector{t}^\star$ can be recovered with \Cref{eq:recover_trans}.
Crucially, our approach is \emph{certifiable} because a duality gap (i.e., the difference between the primal cost and the dual cost) of zero for a candidate solution pair $\Vector{\nu}^\star, \Vector{r}^\star$ is a post-hoc guarantee or certificate of its global optimality.

\section{SDP Tightness and Stability} \label{sec:theory}
In this section, we derive sufficient conditions for our convex relaxation-based approach to hand-eye calibration to be \emph{tight}, ensuring that a certifiably globally optimal solution to the primal problem can be extracted from the solution to its convex relaxation. Throughout this section, we will be dealing with a slightly modified version of \Cref{prob:qcqp_formulation}:
\begin{problem}[Orthogonal Relaxation of Hand-Eye Calibration] \label{prob:qcqp_orthogonal}
		\begin{equation}
\begin{aligned}
\min_{\Matrix{R}, \Vector{t}, \alpha} \; & J_{\Matrix{t}} + J_{\Matrix{R}}\\
\text{\emph{s.t.}} \; & \Matrix{R} \in \LieGroupO{3},
\end{aligned}
\end{equation}
where the cost function $J_{\Matrix{t}} + J_{\Matrix{R}}$ is the same as \Cref{prob:qcqp_formulation}'s and the feasible set is the group of orthogonal matrices $\LieGroupO{3} = \{\Matrix{R}: \Matrix{R}\Matrix{R}^\T = \Identity\}$.  
\end{problem}
\Cref{prob:qcqp_orthogonal} is a relaxation of \Cref{prob:qcqp_formulation} because we allow reflections in addition to rotations. An important consequence of this fact is that any instance of \Cref{prob:qcqp_orthogonal} which exhibits strong duality with a minimizer in $\LieGroupSO{3}$ will also exhibit strong duality if its constraints are `tightened' to limit its feasible set to $\LieGroupSO{3}$ (i.e., made into an instance of \Cref{prob:qcqp_formulation}). The smaller constraint set of \Cref{prob:qcqp_orthogonal} will make the proofs of strong duality simpler, and they will double as proofs of any formulation that adds constraints (e.g., \Cref{prob:qcqp_formulation} with or without redundant orthogonality constraints).
We will also make use of the known-scale version of hand-eye calibration found in \cite{giamou2019_certifiably}:
\begin{problem}[Hand-Eye Calibration with Known Scale] \label{prob:qcqp_known_scale}
		\begin{equation}
\begin{aligned}
\min_{\Matrix{R}, \Vector{t}} \; & J'_{\Matrix{t}} + J_{\Matrix{R}}\\
\text{\emph{s.t.}} \; & \Matrix{R} \in \LieGroupO{3},
\end{aligned}
\end{equation}
where 
\begin{equation}
	J'_{\Matrix{t}} = \sum_{t=1}^T \norm{\Matrix{R}\Vector{t}_{a_t} + \Vector{t}
 - \Matrix{R}_{b_t}\Vector{t} - \Vector{t}_{b_t}}{2}^2,
\end{equation}
which is identical to the translation cost terms in \Cref{prob:qcqp_formulation} except the scale $\alpha$ is omitted (i.e., it is assumed to be known and equal to 1). 
\end{problem}
The `known scale' case in \Cref{prob:qcqp_known_scale}, which applies to scenarios where both sensors measure their egomotion without scale ambiguity, is simpler to work with and will be used to extend the proofs in this section to the more general, monocular  formulation of \Cref{prob:qcqp_formulation}.

\begin{proposition}[Tightness of the Noise-Free Case] \label{prop:zero_duality_gap_known_scale}
Any noise-free instance of \Cref{prob:qcqp_orthogonal} exhibits strong duality (i.e., the duality gap is zero and the optimal solution can be extracted from its dual SDP solution).
\end{proposition}
\begin{proof}
	We will use Lemma 2 from \cite{cifuentes_local_2018}, which requires that we present a primal solution $\Vector{x}^\star \in \Real^n$ and a dual solution $\Matrix{\lambda}^\star \in \Real^m$ that satisfy  
	\begin{enumerate}
		\item primal feasibility ($h(\Vector{x}^\star = \Vector{0}$)),
		\item dual feasibility ($\LagrangianHessian(\Vector{\lambda^\star}) \succeq 0$),
		\item and complementarity ($\Vector{\lambda}^{\star\T} \nabla h(\Vector{x}^\star) = -\nabla f(\Vector{x}^\star)$),
	\end{enumerate}
	where $\LagrangianHessian$ is the Hessian of the Lagrangian of \Cref{prob:qcqp_orthogonal}.
	Selecting the ground truth rotation $\Vector{R}^\star$ and translation $\Vector{t}^\star$ clearly satisfy primal feasibility since $\Vector{R}^\star$ is a rotation matrix. If we select $\Matrix{\lambda}^\star = \Vector{0}$, we get $\LagrangianHessian(\Vector{\lambda}^\star) = \nabla^2 f$ which is clearly positive semidefinite because $f$ is a positive sum of convex norms and therefore convex. Finally, since the ground truth values of $\Vector{R}$ and $\Vector{t}$ give a cost of zero and $f \geq 0$, they are an unconstrained minimizer (i.e., $\nabla f(\Vector{x}^\star) = \Vector{0})$. Therefore, 
	\begin{equation}
		\Matrix{\lambda}^\T \nabla h(\Vector{x}^\star) = 0 = -\nabla f(\Vector{x}^\star).
	\end{equation}
\end{proof}
\noindent This result trivially extends to \Cref{prob:qcqp_known_scale} and \Cref{prob:qcqp_formulation} with any number of redundant constraints: 
\begin{corollary}[Known Scale and Redundant Constraints] \label{cor:known_scale}
	The known scale case (\Cref{prob:qcqp_known_scale}), and any formulation with redundant constraints, also exhibit strong duality for all noise-free instances. 
\end{corollary}

Next, we present \Cref{lem:strict_convexity}, which will be crucial for proving our main results in Theorems \ref{thm:stability_known_scale} and \ref{thm:stability_unknown_scale}. One of the lemma's conditions, which essentially requires the sensor platform to rotate about two distinct axes in a fixed global reference frame, is a common observability criterion found in similar formulations of extrinsic calibration \cite{brookshire2013extrinsic, Kelly_2011}. 

\begin{lemma}[Strict Convexity] \label{lem:strict_convexity}
A noise-free instance of our extrinsic calibration from egomotion \Cref{prob:qcqp_known_scale} has a strictly convex cost $f = J'_{\Matrix{t}} + J_{\Matrix{R}}$ if the measurement data is generated by motion that includes rotations $\Vector{R}_{b_i} \neq \Vector{I}, \Vector{R}_{b_j} \neq \Vector{I}$ of the sensor platform about two unique axes with corresponding translations $\Vector{t}_{a_i}, \Vector{t}_{a_j}$ satisfying 

\begin{equation} \label{eq:span_condition}
	[(\Vector{R}\Vector{t}_{a_i})^\T \ (\Vector{R}\Vector{t}_{a_j})^\T]^\T \notin \Span(\Col(\begin{bmatrix}
		\Vector{I} - \Vector{R}_{b_i} \\
		\Vector{I} - \Vector{R}_{b_j}
		\end{bmatrix})),
\end{equation}
where $\Matrix{R}$ is the ground truth extrinsic rotation. 
\end{lemma}

\begin{proof}
	Our proof relies on the fact that a positive definite Hessian matrix $\nabla^2 f$ is a necessary and sufficient condition for strict convexity of quadratic functions \cite{boyd2004convex}. We first note that $f$ is a positively weighted sum of quadratic norms, meaning $f = \Vector{x}^\T\Vector{Q}\Vector{x} + \Vector{b}^\T\Vector{x} + c \geq 0$ and $\nabla^2 f = \Matrix{Q} \succeq 0$. We then note that the first-order optimality condition gives us 
	\begin{equation}
		\nabla f|_{\Vector{x}^\star} = 0 \implies \Vector{Q}\Vector{x}^\star = -\Vector{b}.
	\end{equation}
	This tells us that if the solution $\Vector{x}^\star$ is unique, then $\Vector{Q} \succ 0$, otherwise the nullspace of $\Vector{Q}$ provides infinite solutions. Therefore, it suffices to show that the global (unconstrained) minimizer $\Vector{x}^\star$ is unique in order to prove that $f$ is strictly convex. Since $f \geq 0$ is a sum of squared residuals that are all equal to zero when $\Vector{x}^\star = [\Vector{t}^{\star\T} \ \Vector{r}^{\star\T}]^\T$, where $\Vector{t}^\star$ and $\Vector{r}^\star$ are the true calibration parameters used to generate the noise-free measurements, we see that    
		$f([\Vector{t}^{\star\T} \ \Vector{r}^{\star\T}]^\T) = 0$ is a global minimum of $f$. Thus, we must show that $f(\Vector{x}) = 0$ implies that $\Vector{x} = \Vector{x}^\star$. 
		
		We will first demonstrate that $J_{\Matrix{R}}(\Vector{r}) = 0$ if and only if $\Vector{r} = \gamma \Vector{r}^\star, \ \gamma \in \Real$. To accomplish this, we will note that each squared residual term $\Vector{R}\Vector{R}_{a_t}-\Vector{R}_{b_t}\Vector{R}$ is equal to zero if and only if
		\begin{equation}
			\Vector{R} = \Vector{R}_{b_t}\Vector{R}\Vector{R}_{a_t}^\T. 
		\end{equation} 
		Using the Kronecker product's vectorization identity yields
		\begin{equation} \label{eq:rotation_cost_kronecker}
			(\Identity_9 - \Vector{R}_{a_t} \otimes \Vector{R}_{b_t})\Vector{r} = \Vector{0}_9, \ t \in \{i, j\}.
		\end{equation}
		Since $\Matrix{R}_{b_i}$ and $\Matrix{R}_{b_j}$ are rotations about distinct axes, Lemma 1 from \cite{andreff_robot_2001} ensures that the system in \Cref{eq:rotation_cost_kronecker} has a unique solution up to scale $\gamma$.
		
		Having established that $f(\Vector{x}) = 0 \implies \Vector{r} = \gamma\Vector{r}^\star$, we will now investigate the squared residual terms of $J_{\Matrix{t}}$:
		\begin{equation}
			\Vector{R}\Vector{t}_{a_t} + \Vector{t} - \Vector{R}_{b_i}\Vector{t} - \Vector{t}_{b_t}.
		\end{equation}
		Substituting in $\gamma\Vector{r}^\star$ and setting the $i$th residual to zero gives
		\begin{equation}
			(\Vector{I} - \Vector{R}_{b_i})\Vector{t} + \gamma \Vector{R}^\star \Vector{t}_{a_i} = \Vector{t}_{b_i},
		\end{equation}
		which we can rearrange and combine with the $j$th residual to get 
		\begin{equation}
			\begin{bmatrix}
				\Vector{I} - \Vector{R}_{b_i} & \Vector{R}^\star \Vector{t}_{a_i} \\
				\Vector{I} - \Vector{R}_{b_j} & \Vector{R}^\star \Vector{t}_{a_j}
			\end{bmatrix} 
			\begin{bmatrix}
				\Vector{t} \\ 
				\gamma 
			\end{bmatrix}
			= \Matrix{M} \begin{bmatrix}
				\Vector{t} \\ 
				\gamma 
			\end{bmatrix} = \begin{bmatrix}
				\Vector{t}_{b_i} \\
				\Vector{t}_{b_j}
			\end{bmatrix}. 
		\end{equation}
		In order to prove that $\Vector{t} = \Vector{t}^\star, \gamma = 1$ is a unique solution, we must demonstrate that $\Matrix{M} \in \Real^{6\times 4}$ is full rank. First, note that 
		\begin{equation}
			\Rank(\begin{bmatrix}
				\Vector{I} - \Vector{R}_{b_i} \\
				\Vector{I} - \Vector{R}_{b_j} 
			\end{bmatrix})
			= 3.
		\end{equation}
		Suppose $\exists\,\Vector{w} \neq \Vector{0}$ such that $(\Vector{I} - \Vector{R}_{b_i})\,\Vector{w} = 0$ and $(\Vector{I} - \Vector{R}_{b_j})\Vector{w} = 0$. This would mean that $\Vector{w}$ is along the axis of rotation for both $\Vector{R}_{b_i}$ and $\Vector{R}_{b_j}$, which contradicts our assumption. Since the left three columns of $\Vector{M}$ are rank 3, $\Vector{M}$ is rank 4 when the fourth column $[(\Vector{R}^\star \Vector{t}_{a_t})^\T (\Vector{R}^\star \Vector{t}_{a_j})^\T]^\T$ is not in the span of of the first 3 columns. This is precisely the condition 
		\begin{equation}
		[(\Vector{R}\Vector{t}_{a_t})^\T \ (\Vector{R}\Vector{t}_{a_j})^\T]^\T \notin \Span(\Col(\begin{bmatrix}
				\Vector{I} - \Vector{R}_{b_i} \\
				\Vector{I} - \Vector{R}_{b_j}
			\end{bmatrix})),
		\end{equation}
		which is one of our assumptions. Therefore, $\Vector{x}^\star$ is a unique minimizer of $f$ and $f$ is strictly convex.
\end{proof}

We are now prepared to prove our two main results. These theorems demonstrate that the strong duality of our formulation of extrinsic calibration is inherently stable to the addition of measurement noise. Both theorems rely on technical results described in detail in \cite{cifuentes_local_2018}. 
\begin{theorem}[Stability of \Cref{prob:qcqp_known_scale}] \label{thm:stability_known_scale}
	 Let $\Vector{\theta}$ be a vector containing all the egomotion measurements that parameterize the cost function of \Cref{prob:qcqp_orthogonal}. Let $\bar{\Vector{\theta}}$ be any parameterization such that the conditions of \Cref{lem:strict_convexity} hold (i.e., they describe a noise-free problem instance with rotation about two unique axes and Condition \ref{eq:span_condition} holds). There exists some $\epsilon > 0$ such that if $\Norm{\Vector{\theta} - \bar{\Vector{\theta}}} \leq \epsilon$, then strong duality holds for the instance of \Cref{prob:qcqp_orthogonal} described by $\Vector{\theta}$, and the global optimum can be obtained via the solution of the dual problem. 
\end{theorem}
\begin{proof}
	We will use Theorem 8 from \cite{cifuentes_local_2018}, which requires that:
	\begin{enumerate}
		\item the cost function $f$ varies continuously as a function of $\Vector{\theta}$,
		\item $\bar{\Vector{\theta}}$ is such that $f_{\bar{\Vector{\theta}}}$ is strictly convex (where $f_{\Vector{\theta}}$ simply denotes the specific cost function formed with measurements in $\Vector{\theta}$), 
		\item \Cref{prob:qcqp_orthogonal}'s minimizer $\Vector{x}^\star$ is also the minimizer of the unconstrained cost function $f_{\bar{\Vector{\theta}}}$ (i.e., $\nabla f_{\bar{\Vector{\theta}}}(\Vector{x}^\star) = \Matrix{0}$), and
		\item the Abadie constraint qualification (ACQ) holds for the algebraic variety described by the constraints of \Cref{prob:qcqp_orthogonal}.  
	\end{enumerate}
	 The cost function $f_{\Vector{\theta}}$ depends quadratically (and therefore continuously) on $\Vector{\theta}$, satisfying condition 1). \Cref{lem:strict_convexity} ensures that 2) and 3) hold. Finally, condition 4) holds because the variety described by $\LieGroupO{3}$ defines a radical ideal (see Lemma 21 and Examples 7.3-4 in \cite{cifuentes_local_2018} for details).
\end{proof}

\begin{theorem}[Stability of \Cref{prob:qcqp_orthogonal}] \label{thm:stability_unknown_scale}
	The stability property described in \Cref{thm:stability_known_scale} also holds for the unknown scale case of \Cref{prob:qcqp_orthogonal}.
\end{theorem}
\begin{proof}
	Since the cost function of \Cref{prob:qcqp_orthogonal} is homogeneous, it is not strictly convex for any $\Vector{\theta}$ and we cannot use the same approach as in \Cref{thm:stability_known_scale}. We will instead use Remark 7 and the more general Theorem 14 from \cite{cifuentes_local_2018}, which require that:
	\begin{enumerate}
		\item ACQ holds at $\Vector{x}^\star$ of the problem instance described by $\Vector{\theta}$,
		\item the constraints of \Cref{prob:qcqp_orthogonal} describe a smooth manifold, and
		\item the Hessian of the Lagrangian of the cost function $f_{\bar{\Vector{\theta}}}$ is corank-one.
	\end{enumerate}
	We saw in \Cref{thm:stability_known_scale} that 1) holds, and $\LieGroupSO{3}$ is a smooth manifold that satisfies 2). Therefore, it remains to prove 3).
	In \Cref{prop:zero_duality_gap_known_scale} and \Cref{cor:known_scale}, we established that exact measurements lead to a zero-duality gap problem instance with corresponding Lagrange multiplier $\Vector{\lambda} = \Vector{0}$ for both \Cref{prob:qcqp_known_scale} and \Cref{prob:qcqp_orthogonal}. \Cref{lem:strict_convexity} established that the affine (i.e., non-homogenized) cost function of the known scale problem is strictly convex. In other words, its Hessian $\Matrix{H} = \nabla^2 f$ is positive definite. We can use the fact that its Lagrange multiplier $\Vector{\lambda}$ is zero to conclude that its Lagrangian's Hessian $\LagrangianHessian_{\bar{\Vector{\theta}}}(\Vector{\lambda}) = \Matrix{H}$ is also positive definite. Letting $\Vector{y}$ denote the affine coordinates of \Cref{prob:qcqp_known_scale}, we note that the homogeneous \Cref{prob:qcqp_orthogonal}'s coordinates are of the form
	\begin{equation}
		\Vector{x} = \bbm\Vector{y}^\T \!\!& \alpha\ebm^\T,
	\end{equation}
	where $\alpha$ is the scale variable. The Hessian of \Cref{prob:qcqp_orthogonal}'s cost function can therefore be written
	\begin{equation}
		\Matrix{G} = \begin{bmatrix}
			\Matrix{H} & \frac{1}{2} \Vector{b} \\
			\frac{1}{2} \Vector{b}^\T & c
		\end{bmatrix} \in \Real^{N\times N},
	\end{equation} 
	for appropriate parameters $\Vector{b}$ and $c$. We need to demonstrate (for Theorem 14 in \cite{cifuentes_local_2018}) that $\Matrix{G}$ has corank-one (i.e., rank $N-1$). Elementary properties of matrix rank tell us that 
	\begin{equation}
		\text{rank}(\Matrix{G}) \geq
		\text{rank}(\begin{bmatrix}
			\Matrix{H} \\
			\frac{1}{2} \Vector{b}^\T
		\end{bmatrix}) \geq \text{rank}(\Matrix{H}) = N-1,
	\end{equation} 
	with the final equality coming from the positive definiteness of $\Matrix{H}$. Since the cost function for the unknown scale problem is homogeneous and has a minimizer $\Vector{x}^\star \neq \Vector{0}$ that evaluates to zero, we know it is not full rank. Therefore, the Hessian $\Matrix{G}$ is corank-one as required.  
\end{proof}

Theorems \ref{thm:stability_known_scale} and \ref{thm:stability_unknown_scale} demonstrate the existence of a measurement error bound $\epsilon$ within which our hand-eye calibration formulations exhibit strong duality, but we leave the exact quantification of this bound as future work.

\section{Experiments and Results} 
\label{sec:experiments}

In order to demonstrate the strong duality guarantees of \Cref{sec:theory}, we focus primarily on synthetic experiments where measurement statistics and the ground truth value of $\Matrix{\Theta}$ are known exactly. 
We also compare our convex relaxation approach against a simple method that does not guarantee a global minimum on the same synthetic data.
Throughout this section, `optimality' refers specifically to global optimality of an extrinsic transformation estimate with respect to the cost function of \Cref{prob:qcqp_formulation}.
In the presence of noise, optimality does not imply zero error with respect to the ground truth $\Matrix{\Theta}$; e.g., the dual solutions in \Cref{fig:Error_plot} are globally optimal with respect to \Cref{prob:qcqp_formulation} but still differ from ground truth. Throughout all experiments, the runtime of our algorithm was on the order of two seconds without tuning optimization parameters, which is similar to the performance of its predecessor as reported in \cite{giamou2019_certifiably}.

\subsection{Optimality Certification Conditions}

Throughout our experiments, three criteria are used to certify that the solution to the dual problem is optimal. First, a singular value decomposition (SVD) is performed to evaluate the numerical rank of the solution matrix $\Matrix{Z}$ in \Cref{prob:dual_relaxation}. Any right-singular vector with corresponding singular value less than $10^{-3}$ is used to form the solution to the primal problem. Next, the extracted rotation solutions are checked via $\norm{\Matrix{R}^T\Matrix{R}-\Identity}{F}<10^{-3}$, which ensures that the solution belongs to $\LieGroupSO{3}$. Finally, solutions with a duality gap greater than $0.01\%$ of the primal cost are rejected.\footnote[3]{The approximations of floating point arithmetic necessitate the use of numerical tolerances for `zero' singular values and duality gaps; $10^{-3}$ and $0.01\%$ performed well experimentally for all problem instances tested.}

\subsection{Synthetic Data}
\label{sec:synthetic_data}

The simulation data were created by generating trajectories on a smooth, undulating surface. The $x$-axis of $\CoordinateFrame{a_t}$ was set to be tangent to the trajectory at every point $\Vector{t}_{w}^{a_t w}$, while the\linebreak $z$-axis was set to be normal to the surface, thus defining $\Matrix{R}_{wa_t}$. The full pose $\Matrix{T}_{wa_t}$ is therefore defined by $\Matrix{R}_{wa_t}$ and the position $\Vector{t}_{w}^{a_t w}$. The absolute position of the second sensor $\Vector{t}_{w}^{b_t w}$ at each time step was determined using the ground-truth value of the extrinsic transformation $\Matrix{\Theta}$ (see \Cref{eq:extrinsic_calibration}). An example ground-truth trajectory produced via this method is shown in Figure \ref{fig:trajectory}.

To create each dataset, egomotion measurements $\Matrix{T}_{a_t}$ and $\Matrix{T}_{b_t}$ were extracted from the trajectories of sensors $a$ and $b$, respectively, where $b$ was the monocular camera. All camera translation vectors were scaled by $\alpha > 0$. Finally, zero-mean Gaussian noise was added to each translation vector and injected into each rotation matrix via a left perturbation of $\LieGroupSO{3}$ \cite{barfoot2017state}. Details about the noise variance are given in \Cref{subsubsec:zdg_rc}.
\begin{figure} 
\centering
\includegraphics[width=0.495\textwidth]{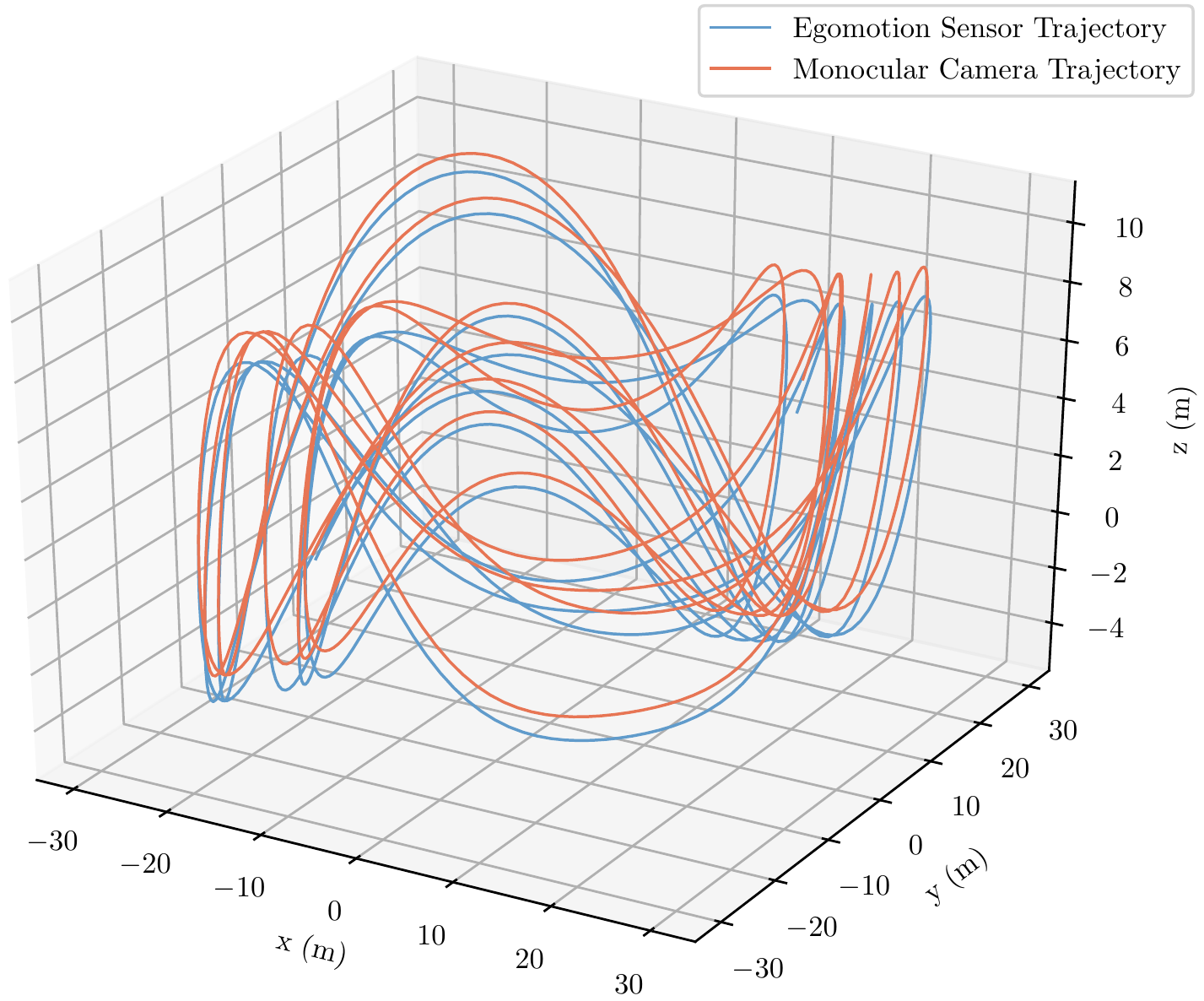}
\caption{Trajectory of the rigid body system described in \Cref{sec:synthetic_data}. Rotations of the rigid body were about all $3$ axes with magnitudes ranging from $0.05$ rad to $0.3$ rad. The $\LieGroupSE{3}$ trajectory of each sensor on the rigid body is corrupted by synthetic noise and the translation estimates of the monocular camera are scaled by $\alpha > 0$. The sensor platform trajectory provides sufficient rotation about different axes to satisfy the assumptions of \Cref{lem:strict_convexity}.}
\label{fig:trajectory}
\vspace{-0.4cm}
\end{figure}

\subsubsection{Zero-Duality Gap and Redundant Constraints}
\label{subsubsec:zdg_rc}

In \Cref{sec:theory}, we proved the stability of our QCQP formulation of hand-eye calibration to noisy measurements. In order to verify this result, the trajectory in \Cref{fig:trajectory} was corrupted with measurement noise and the calibration was determined with our method. Each bar in \Cref{fig:Optimal_plot} represents the percentage of 100 trials for which the solver was able to find (and certify via the duality gap) a globally optimal solution. Each grouping in \Cref{fig:Optimal_plot} corresponds to a different standard deviation of the translational noise. We ran additional tests with rotational standard deviations up to 3 rad. The rotational noise tests revealed that rotational noise, at the magnitudes tested, did not affect solution optimality (although, as noted previously, this does not imply accuracy for very large noise magnitudes). Each bar corresponds to a different combination of $\LieGroupSO{3}$ constraints (\Cref{eq:SO3con}). 

Our first observation is that the default case, labelled `R,' with no redundant or right-handedness constraints, does in fact achieve a global optimum in 100\% of cases, provided that the standard deviation of the translational measurement noise is $1\%$ of the translation magnitude.
This behaviour was predicted by \Cref{thm:stability_unknown_scale}. 
Secondly, we note that including both redundant column and row orthogonality constraints (`R+C') improves the stability of our approach to measurement noise.
Finally, the inclusion of the right-handedness constraints (`R+H' and `R+C+H'), which prohibit orthogonal matrices that include reflections, also increases the robustness of our solver to noise. 
These results mirror those found in \cite{giamou2019_certifiably}, which are (retroactively) predicted by \Cref{thm:stability_known_scale}.

\subsubsection{Calibration Accuracy}

In this section we evaluate the performance of our algorithm against a simple, suboptimal linear approach inspired by \cite{andreff_robot_2001}.
This suboptimal approach uses the same cost function as \Cref{prob:qcqp_reduced}, but solves the unconstrained problem via SVD before projecting onto the nearest orthogonal matrix using the method of \cite{horn1988closed} and then extracting the unconstrained optimal $\Vector{t}$ and $\alpha$ from \Cref{eq:recover_trans}.
Since it solves successive linear least squares systems to approximately minimize the cost function, we refer to this as the linear solution.
Our results are displayed in \Cref{fig:Error_plot}: each row has a different pair of translational and rotational noise settings. 
At low error, both algorithms perform similarly, but as the error increases our solution outperforms the linear solution, highlighting the importance of a globally optimal approach.
\begin{figure} 
\centering
\includegraphics[width=0.485\textwidth]{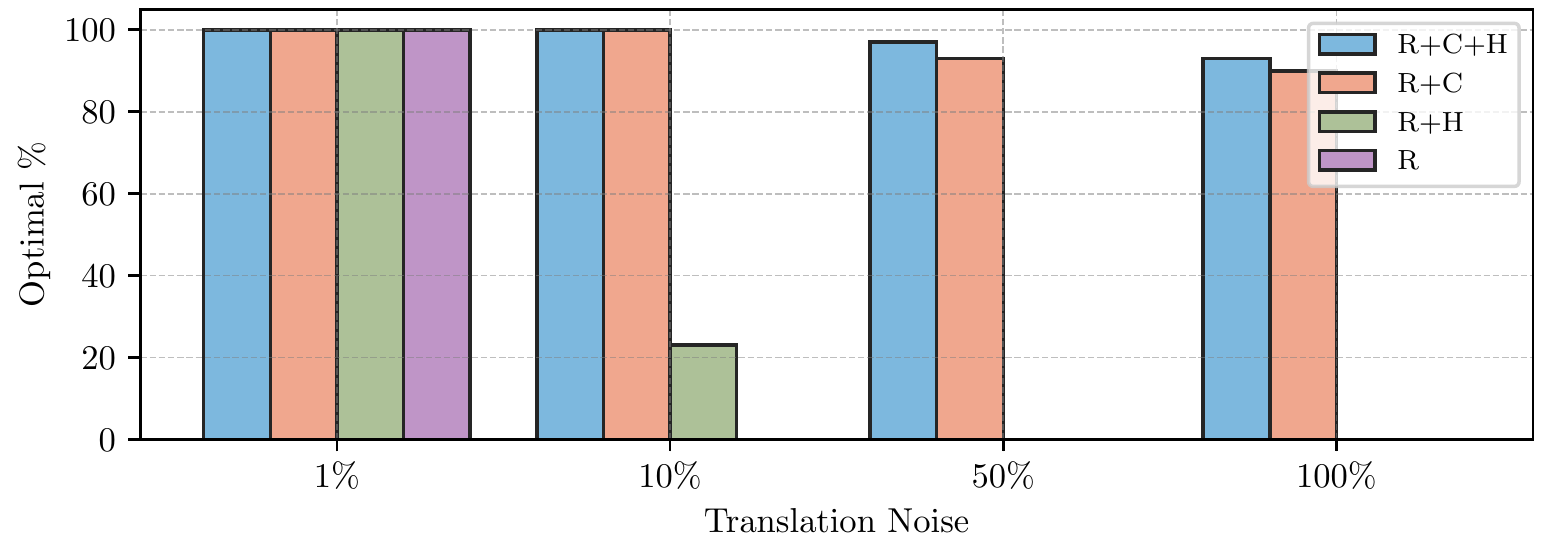}
\caption{A bar plot indicating the percentage of 100 noisy random trials in which the dual SDP solver found the global optimum. Each grouping corresponds to the percentage of translation measurement magnitude used as the standard deviation for the noise added. The legend specifies which combination of $\LieGroupSO{3}$ constraints was used: R is row orthogonality, C is column orthogonality, and H is right-handedness.}
\label{fig:Optimal_plot}
\end{figure}

\begin{figure} 
\centering
\includegraphics[width=0.5\textwidth]{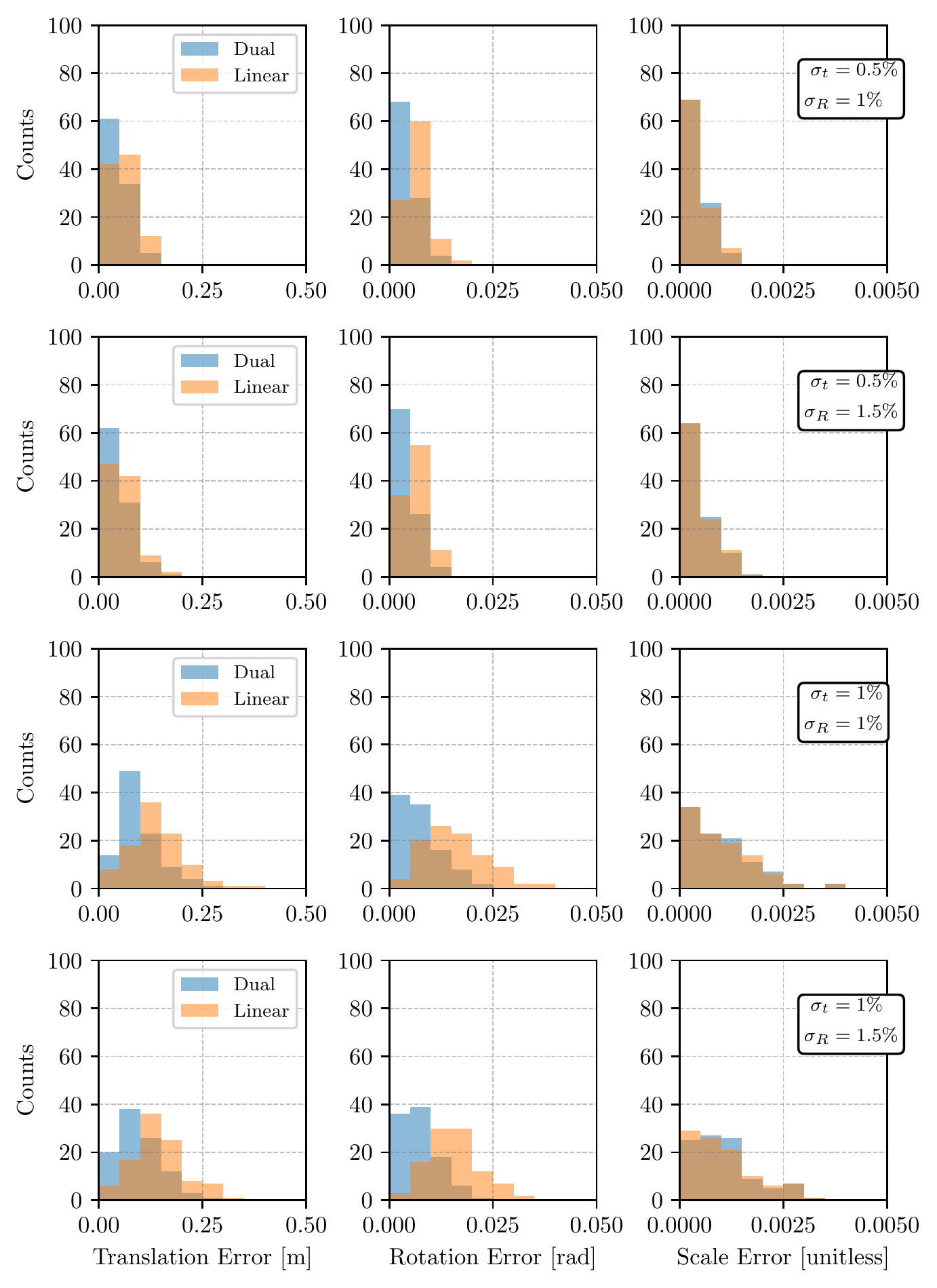}
\caption{Histograms of extrinsic transformation and scale error across simulations with varying noise. Our globally optimal method has far less rotational and translational error than the simple linear approach.}
\label{fig:Error_plot}
\vspace{-0.3cm}
\end{figure}

\section{Conclusion}
\label{sec:conclusion}

We have presented a novel, certifiable hand-eye calibration algorithm with a priori global optimality guarantees. 
Our experiments demonstrated that the zero-duality-gap region is large enough to accommodate severe sensor noise and that our algorithm only takes a few seconds with a generic SDP solver on a laptop. 
This paper focused on the theoretical properties of our algorithm---we intend to include extensive experimentation on real-world datasets in its sequel. 
Additionally, extensions to the closely-related robot-world calibration problem \cite{heller2014hand}, probabilistic cost function variants \cite{ma_probabilistic_2016}, and robust problem formulations \cite{hu_novel_2019} all hold promise for our technique. 
Finally, while the challenge of accurate joint spatiotemporal extrinsic calibration has been explored in the literature \cite{rehder_extending_2016, lambert2017entropy, marr2018unified, hutter_evaluation_2018}, a \emph{certifiable} algorithm has not, to our knowledge, been proposed.

\balance
\bibliographystyle{IEEEtran}
\bibliography{robotics_abbrv,refs}

\begin{thebibliography}{10}
\providecommand{\url}[1]{#1}
\csname url@rmstyle\endcsname
\providecommand{\newblock}{\relax}
\providecommand{\bibinfo}[2]{#2}
\providecommand\BIBentrySTDinterwordspacing{\spaceskip=0pt\relax}
\providecommand\BIBentryALTinterwordstretchfactor{4}
\providecommand\BIBentryALTinterwordspacing{\spaceskip=\fontdimen2\font plus
\BIBentryALTinterwordstretchfactor\fontdimen3\font minus
  \fontdimen4\font\relax}
\providecommand\BIBforeignlanguage[2]{{%
\expandafter\ifx\csname l@#1\endcsname\relax
\typeout{** WARNING: IEEEtran.bst: No hyphenation pattern has been}%
\typeout{** loaded for the language `#1'. Using the pattern for}%
\typeout{** the default language instead.}%
\else
\language=\csname l@#1\endcsname
\fi
#2}}

\bibitem{qilongzhang_extrinsic_2004}
{Qilong Zhang} and R.~Pless, ``\BIBforeignlanguage{en}{Extrinsic calibration of
  a camera and laser range finder (improves camera calibration)},'' in
  \emph{\BIBforeignlanguage{en}{IEEE/RSJ Intl. Conf. Intelligent Robots and
  Systems (IROS)}}, vol.~3, {Sendai, Japan}, 2004, pp. 2301--2306.

\bibitem{gomez-ojeda_extrinsic_2015}
R.~{Gomez-Ojeda}, J.~Briales, E.~{Fernandez-Moral}, and J.~{Gonzalez-Jimenez},
  ``\BIBforeignlanguage{en}{Extrinsic calibration of a {2D} laser-rangefinder
  and a camera based on scene corners},'' in \emph{\BIBforeignlanguage{en}{IEEE
  Intl. Conf. Robotics and Automation (ICRA)}}, {Seattle, USA}, May 2015, pp.
  3611--3616.

\bibitem{brookshire2012automatic}
J.~Brookshire and S.~Teller, ``Automatic calibration of multiple coplanar
  sensors,'' \emph{Robotics: Science and Systems VII}, vol.~33, 2012.

\bibitem{brookshire2013extrinsic}
------, ``Extrinsic calibration from per-sensor egomotion,'' \emph{Robotics:
  Science and Systems VIII}, pp. 504--512, 2013.

\bibitem{giamou2019_certifiably}
M.~{Giamou}, Z.~{Ma}, V.~{Peretroukhin}, and J.~{Kelly}, ``Certifiably globally
  optimal extrinsic calibration from per-sensor egomotion,'' \emph{{IEEE}
  Robot. Autom. Lett.}, vol.~4, no.~2, pp. 367--374, April 2019.

\bibitem{chiuso2002_structure}
A.~{Chiuso}, P.~{Favaro}, {Hailin Jin}, and S.~{Soatto}, ``Structure from
  motion causally integrated over time,'' \emph{{IEEE} Trans. Pattern Anal.
  Machine Intell.}, vol.~24, no.~4, pp. 523--535, April 2002.

\bibitem{cifuentes_local_2018}
D.~Cifuentes, S.~Agarwal, P.~A. Parrilo, and R.~R. Thomas,
  ``\BIBforeignlanguage{en}{On the local stability of semidefinite
  relaxations},'' \emph{\BIBforeignlanguage{en}{arXiv:1710.04287 [math]}}, Dec.
  2018.

\bibitem{heller2014hand}
J.~Heller, D.~Henrion, and T.~Pajdla, ``Hand-eye and robot-world calibration by
  global polynomial optimization,'' in \emph{IEEE Intl. Conf. Robotics and
  Automation (ICRA)}, Hong Kong, China, 2014, pp. 3157--3164.

\bibitem{hu_novel_2019}
X.~Hu, D.~Olesen, and K.~Per, ``\BIBforeignlanguage{en}{A novel robust approach
  for correspondence-free extrinsic calibration},'' in
  \emph{\BIBforeignlanguage{en}{IEEE/RSJ Intl. Conf. Intelligent Robots and
  Systems (IROS)}}, {Macau, China}, Nov. 2019.

\bibitem{tsai_new_1989}
R.~Tsai and R.~Lenz, ``\BIBforeignlanguage{en}{A new technique for fully
  autonomous and efficient {{3D}} robotics hand/eye calibration},''
  \emph{\BIBforeignlanguage{en}{{IEEE} Trans. Robotics}}, vol.~5, no.~3, pp.
  345--358, June 1989.

\bibitem{daniilidis_hand-eye_1999}
K.~Daniilidis, ``\BIBforeignlanguage{en}{Hand-eye calibration using dual
  quaternions},'' \emph{\BIBforeignlanguage{en}{Intl. J. Robotics Research}},
  vol.~18, no.~3, pp. 286--298, Mar. 1999.

\bibitem{horaud1995hand}
R.~Horaud and F.~Dornaika, ``Hand-eye calibration,'' \emph{Intl. J. Robotics
  Research}, vol.~14, no.~3, pp. 195--210, 1995.

\bibitem{andreff_robot_2001}
N.~Andreff, R.~Horaud, and B.~Espiau, ``\BIBforeignlanguage{en}{Robot hand-eye
  calibration using structure-from-motion},''
  \emph{\BIBforeignlanguage{en}{Intl. J. Robotics Research}}, vol.~20, no.~3,
  pp. 228--248, Mar. 2001.

\bibitem{wei2018calibration}
L.~Wei, L.~Naiguang, D.~Mingli, and L.~Xiaoping, ``Calibration-free
  robot-sensor calibration approach based on second-order cone programming,''
  in \emph{MATEC Web of Conferences}, vol. 173.\hskip 1em plus 0.5em minus
  0.4em\relax EDP Sciences, 2018.

\bibitem{walters_robust_2019}
C.~Walters, O.~Mendez, S.~Hadfield, and R.~Bowden, ``\BIBforeignlanguage{en}{A
  robust extrinsic calibration framework for vehicles with unscaled sensors},''
  in \emph{\BIBforeignlanguage{en}{IEEE/RSJ Intl. Conf. Intelligent Robots and
  Systems (IROS)}}, {Macau, China}, Nov. 2019, pp. 36--42.

\bibitem{ma_probabilistic_2016}
Q.~Ma, Z.~Goh, and G.~S.~Chirikjian, ``\BIBforeignlanguage{en}{Probabilistic
  approaches to the {{AXB}} = {{YCZ}} calibration problem in multi-robot
  systems},'' in \emph{\BIBforeignlanguage{en}{Robotics: {{Science}} and
  {{Systems XII}}}}, 2016.

\bibitem{fredriksson2012simultaneous}
J.~Fredriksson and C.~Olsson, ``Simultaneous multiple rotation averaging using
  {Lagrangian} duality,'' in \emph{Asian Conf. Computer Vision (ACCV)}.\hskip
  1em plus 0.5em minus 0.4em\relax Springer, 2012, pp. 245--258.

\bibitem{eriksson2018rotation}
A.~Eriksson, C.~Olsson, F.~Kahl, and T.-J. Chin, ``Rotation averaging and
  strong duality,'' in \emph{IEEE Conf. Computer Vision and Pattern Recognition
  (CVPR)}, Salt Lake City, USA, 2018, pp. 127--135.

\bibitem{yang_teaser_2020}
H.~Yang, J.~Shi, and L.~Carlone, ``\BIBforeignlanguage{en}{{{TEASER}}: {{Fast}}
  and {{Certifiable Point Cloud Registration}}},''
  \emph{\BIBforeignlanguage{en}{arXiv:2001.07715 [cs, math]}}, Jan. 2020.

\bibitem{rosen2016se}
D.~M. Rosen, L.~Carlone, A.~S. Bandeira, and J.~J. Leonard,
  ``\BIBforeignlanguage{en}{{{SE}}-{{Sync}}: {{A}} certifiably correct
  algorithm for synchronization over the special {{Euclidean}} group},''
  \emph{\BIBforeignlanguage{en}{Intl. J. Robotics Research}}, vol.~38, no. 2-3,
  pp. 95--125, Mar. 2019.

\bibitem{briales2017cartan}
J.~Briales and J.~Gonzalez-Jimenez, ``Cartan-sync: Fast and global
  $\mathrm{SE}(d)$-synchronization,'' \emph{{IEEE} Robot. Autom. Lett.},
  vol.~2, no.~4, pp. 2127--2134, 2017.

\bibitem{fan_efficient_nodate}
T.~Fan, H.~Wang, M.~Rubenstein, and T.~Murphey, ``Efficient and guaranteed
  planar pose graph optimization using the complex number representation,'' in
  \emph{IEEE/RSJ Intl. Conf. Intelligent Robots and Systems (IROS)}, Macau,
  China, 2019, pp. 1904--1911.

\bibitem{briales2017convex}
J.~Briales, J.~Gonzalez-Jimenez, \emph{et~al.}, ``Convex global {3D}
  registration with {Lagrangian} duality,'' in \emph{IEEE Conf. Computer Vision
  and Pattern Recognition (CVPR)}, Honolulu, USA, 2017, pp. 5612--5621.

\bibitem{olsson2008solving}
C.~Olsson and A.~Eriksson, ``Solving quadratically constrained geometrical
  problems using {Lagrangian} duality,'' in \emph{Intl. Conf. Pattern
  Recognition (ICPR)}, Tampa, USA, 2008, pp. 1--5.

\bibitem{briales_certifiably_2018}
J.~Briales, L.~Kneip, and J.~{Gonzalez-Jimenez}, ``\BIBforeignlanguage{en}{A
  certifiably globally optimal solution to the non-minimal relative pose
  problem},'' in \emph{\BIBforeignlanguage{en}{IEEE Conf. Computer Vision and
  Pattern Recognition (CVPR)}}, {Salt Lake City, USA}, June 2018, pp. 145--154.

\bibitem{garcia-salguero_certifiable_2020}
M.~{Garcia-Salguero}, J.~Briales, and J.~{Gonzalez-Jimenez},
  ``\BIBforeignlanguage{en}{Certifiable relative pose estimation},''
  \emph{\BIBforeignlanguage{en}{arXiv:2003.13732 [cs]}}, Mar. 2020.

\bibitem{zhao_efficient_2019}
J.~Zhao, ``\BIBforeignlanguage{en}{An efficient solution to non-minimal case
  essential matrix estimation},''
  \emph{\BIBforeignlanguage{en}{arXiv:1903.09067 [cs]}}, Oct. 2019.

\bibitem{ivancevic2011lecture}
V.~G. Ivancevic and T.~T. Ivancevic, ``Lecture notes in {Lie} groups,'' 2011.

\bibitem{fackler2005notes}
P.~L. Fackler, ``Notes on matrix calculus,'' \emph{North Carolina State
  University}, 2005.

\bibitem{boyd2004convex}
S.~Boyd and L.~Vandenberghe, \emph{Convex optimization}.\hskip 1em plus 0.5em
  minus 0.4em\relax Cambridge university press, 2004.

\bibitem{Briales_2017}
J.~{Briales} and J.~{Gonzalez-Jimenez}, ``Convex global {3D} registration with
  {Lagrangian} duality,'' in \emph{IEEE Conf. Computer Vision and Pattern
  Recognition (CVPR)}, Honolulu, USA, 2017, pp. 5612--5621.

\bibitem{SeDuMi}
J.~F. Sturm, ``Using {SeDuMi 1.02}, a {MATLAB} toolbox for optimization over
  symmetric cones,'' \emph{Optimization Methods and Software}, vol.~11, no.
  1-4, pp. 625--653, 1999.

\bibitem{andersen_mosek_2000}
E.~D. Andersen and K.~D. Andersen, ``\BIBforeignlanguage{en}{The {Mosek}
  interior point optimizer for linear programming: An implementation of the
  homogeneous algorithm},'' in \emph{\BIBforeignlanguage{en}{High {{Performance
  Optimization}}}}, {Boston, MA}, 2000, vol.~33, pp. 197--232.

\bibitem{toh_sdpt3_1999}
K.~C. Toh, M.~J. Todd, and R.~H. T{\"u}t{\"u}nc{\"u},
  ``\BIBforeignlanguage{en}{{{SDPT3}} \textemdash{} {{A MATLAB}} software
  package for semidefinite programming, {{Version}} 1.3},''
  \emph{\BIBforeignlanguage{en}{Optimization Methods and Software}}, vol.~11,
  no. 1-4, pp. 545--581, Jan. 1999.

\bibitem{Kelly_2011}
J.~Kelly and G.~S. Sukhatme, ``Visual-inertial sensor fusion: Localization,
  mapping and sensor-to-sensor self-calibration,'' \emph{Intl. J. Robotics
  Research}, vol.~30, no.~1, pp. 56--79, 2011.

\bibitem{barfoot2017state}
T.~D. Barfoot, \emph{State estimation for robotics}.\hskip 1em plus 0.5em minus
  0.4em\relax Cambridge University Press, 2017.

\bibitem{horn1988closed}
B.~K. Horn, H.~M. Hilden, and S.~Negahdaripour, ``Closed-form solution of
  absolute orientation using orthonormal matrices,'' \emph{J. of the Optical
  Society of America A}, vol.~5, no.~7, pp. 1127--1135, 1988.

\bibitem{rehder_extending_2016}
J.~Rehder, J.~Nikolic, T.~Schneider, T.~Hinzmann, and R.~Siegwart,
  ``\BIBforeignlanguage{en}{Extending {Kalibr}: {{Calibrating}} the extrinsics
  of multiple {{IMUs}} and of individual axes},'' in
  \emph{\BIBforeignlanguage{en}{IEEE Intl. Conf. Robotics and Automation
  (ICRA)}}, {Stockholm, Sweden}, May 2016, pp. 4304--4311.

\bibitem{lambert2017entropy}
J.~Lambert, L.~Clement, M.~Giamou, and J.~Kelly, ``Entropy-based
  $\mathrm{Sim(3)}$ calibration of {2D} lidars to egomotion sensors.''

\bibitem{marr2018unified}
J.~Marr and J.~Kelly, ``Unified spatiotemporal calibration of monocular cameras
  and planar lidars,'' in \emph{Proceedings of the 2018 International Symposium
  on Experimental Robotics}.\hskip 1em plus 0.5em minus 0.4em\relax Cham:
  Springer, 2020, vol.~11, pp. 781--790.

\bibitem{hutter_evaluation_2018}
F.~Furrer, M.~Fehr, T.~Novkovic, H.~Sommer, I.~Gilitschenski, and R.~Siegwart,
  ``\BIBforeignlanguage{en}{Evaluation of combined time-offset estimation and
  hand-eye calibration on robotic datasets},'' in
  \emph{\BIBforeignlanguage{en}{Field and Service Robotics (FSR)}}, {Cham},
  2018, vol.~5, pp. 145--159.

\end{thebibliography}

\end{document}